\documentclass[sigconf]{acmart}

\copyrightyear{2024}
\acmYear{2024}
\setcopyright{rightsretained}
\acmConference[GECCO '24]{Genetic and Evolutionary Computation Conference}{July 14--18, 2024}{Melbourne, VIC, Australia}
\acmBooktitle{Genetic and Evolutionary Computation Conference (GECCO '24), July 14--18, 2024, Melbourne, VIC, Australia}
\acmDOI{10.1145/3638529.3654001}
\acmISBN{979-8-4007-0494-9/24/07}

\usepackage{amsmath,amsfonts,bm}

\def\eqref#1{equation~\ref{#1}}

\def\1{\bm{1}}

\def\vzero{{\bm{0}}}

\def\vmu{{\bm{\mu}}}
\def\vtheta{{\bm{\theta}}}
\def\vphi{{\bm{\phi}}}
\def\va{{\bm{a}}}
\def\vb{{\bm{b}}}

\def\vk{{\bm{k}}}

\def\vp{{\bm{p}}}

\def\vs{{\bm{s}}}

\def\vx{{\bm{x}}}
\def\vy{{\bm{y}}}

\def\mA{{\bm{A}}}

\def\mI{{\bm{I}}}

\def\mX{{\bm{X}}}
\def\mY{{\bm{Y}}}

\def\mSigma{{\bm{\Sigma}}}

\DeclareMathAlphabet{\mathsfit}{\encodingdefault}{\sfdefault}{m}{sl}
\SetMathAlphabet{\mathsfit}{bold}{\encodingdefault}{\sfdefault}{bx}{n}

\def\gB{{\mathcal{B}}}

\newcommand{\R}{\mathbb{R}}

\usepackage{amsthm}
\usepackage{url}
\usepackage[utf8]{inputenc} %
\usepackage[T1]{fontenc}    %
\usepackage{hyperref}       %
\usepackage{cleveref}
\usepackage{url}            %
\usepackage{booktabs}       %
\usepackage{amsfonts}       %
\usepackage{nicefrac}       %
\usepackage{microtype}      %
\usepackage{xcolor}         %
\usepackage{soul}
\usepackage{multirow}
\usepackage[toc,page]{appendix}
\usepackage{multicol}
\usepackage{xspace}
\usepackage{amsmath}

\usepackage{siunitx}
\usepackage{graphicx}
\usepackage{subcaption}
\usepackage{booktabs}       %

\usepackage{mathtools}

\usepackage{comment}

\usepackage[ruled,commentsnumbered,linesnumbered,noend]{algorithm2e} 
\DontPrintSemicolon
\SetArgSty{textrm} %

\usepackage{bm}

\usepackage{xcolor}
\usepackage{amsthm}
\newtheorem{theorem}{Theorem}[section]

\newtheorem{definition}[theorem]{Definition}
\newtheorem{conjecture}[theorem]{Conjecture}
\newtheorem{problem}{Problem}[section]

\usepackage{enumitem}

\usepackage{algorithm2e}

\newcommand{\xxnote}[3]{}
\ifx\hidenotes\undefined
  \usepackage{color}
  \renewcommand{\xxnote}[3]{\color{#2}{#1: #3}}
\fi

\usepackage{array}
\usepackage{booktabs}
\usepackage{multirow}
\usepackage{adjustbox}
\newcolumntype{X}[2]{%
    >{\adjustbox{angle=#1,lap=\width-(#2)}\bgroup}%
    c%
    <{\egroup}%
}
\newcommand*\rot{\multicolumn{1}{X{90}{1em}}}

\begin{document}

\title{Density Descent for Diversity Optimization}

\author{David H. Lee}
\orcid{0009-0003-1927-5039}
\email{dhlee@usc.edu}
\affiliation{%
  \institution{University of Southern California}
  \city{Los Angeles}
  \state{California}
  \country{USA}
}

\author{Anishalakshmi V. Palaparthi}
\orcid{0009-0003-9163-4075}
\email{palapart@usc.edu}
\affiliation{%
  \institution{University of Southern California}
  \city{Los Angeles}
  \state{California}
  \country{USA}
}

\author{Matthew C. Fontaine}
\orcid{0000-0002-9354-196X}
\email{mfontain@usc.edu}
\affiliation{%
  \institution{University of Southern California}
  \city{Los Angeles}
  \state{California}
  \country{USA}
}

\author{Bryon Tjanaka}
\orcid{0000-0002-9602-5039}
\email{tjanaka@usc.edu}
\affiliation{%
  \institution{University of Southern California}
  \city{Los Angeles}
  \state{California}
  \country{USA}
}

\author{Stefanos Nikolaidis}
\orcid{0000-0002-8617-3871}
\email{nikolaid@usc.edu}
\affiliation{%
  \institution{University of Southern California}
  \city{Los Angeles}
  \state{California}
  \country{USA}
}

\renewcommand{\shortauthors}{Lee et al.}

\begin{abstract}
Diversity optimization seeks to discover a set of solutions that elicit diverse features.
Prior work has proposed Novelty Search (NS), which, given a current set of solutions, seeks to expand the set by finding points in areas of low density in the feature space.
However, to estimate density, NS relies on a heuristic that considers the $k$-nearest neighbors of the search point in the feature space, which yields a weaker stability guarantee.
We propose Density Descent Search (DDS), an algorithm that explores the feature space via CMA-ES on a continuous density estimate of the feature space that also provides a stronger stability guarantee.
We experiment with DDS and two density estimation methods: kernel density estimation (KDE) and continuous normalizing flow (CNF).
On several standard diversity optimization benchmarks, DDS outperforms NS, the recently proposed MAP-Annealing algorithm, and other state-of-the-art baselines.
Additionally, we prove that DDS with KDE provides stronger stability guarantees than NS, making it more suitable for adaptive optimizers.
Furthermore, we prove that NS is a special case of DDS that descends a KDE of the feature space.

\end{abstract}



\maketitle

\section{Introduction} \label{sec:introduction}

\begin{figure}
    \centering
    \includegraphics[width=\linewidth]{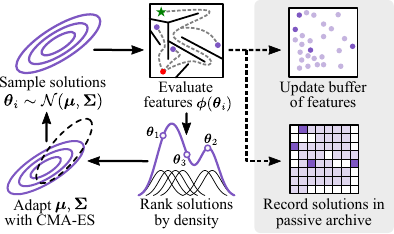}
    \caption{We propose density descent search (DDS) for solving diversity optimization (DO) problems. DDS first draws solutions from a Gaussian $\mathcal{N}(\vmu, \mSigma)$. After computing the solution features (in this case, the final position of the robot in a maze), DDS ranks solutions by density.
    This density ranking is passed to CMA-ES, which updates the search distribution to sample solutions with lower density on the next iteration.
    Concurrently, solutions are stored in a buffer that forms the basis for density estimates, and in a passive archive that tracks all discovered solutions.}
    \label{fig:diagram}
\end{figure}

We study how stable, continuous approximations of density can accelerate the search for diverse solutions to optimization problems.

There is a range of applications where searching directly for behavioral diversity results in finding suboptimal solutions that act as ``stepping stones,'' mitigating convergence to local optima~\citep{gaier2019are}.
A classic example is the problem of training an agent to reach a target position in a  \emph{deceptive maze}~\citep{lehman2011abandoning}.
There, directly minimizing the distance between the agent's final position and a target goal causes the agent to get stuck.
On the other hand, the problem can be solved by ignoring the objective of directly reaching the target and instead attempting to find a diverse range of agents, each of which reaches a different region of the maze.

We refer to the problem of finding a range of solutions that are diverse with respect to prespecified features as diversity optimization (DO). We characterize DO as a special instance of quality diversity optimization (QD)~\citep{pugh2016qd}.
Like DO, QD seeks a diverse set of solutions, but QD also considers an objective function that the solutions must optimize.
For instance, in deceptive maze, one could add an energy consumption objective so that the goal is to find a set of agents that minimize energy consumption while moving to diverse regions in the maze.
Thus, DO is an instance of QD where the objective value of all solutions is constant.

In DO and QD, discovering solutions with new features is challenging because the mapping from solutions to features is complex and non-invertible~\citep{chenu2021selection}.
For example, in maze exploration, it is not known \emph{a priori} how to produce an agent that travels to a given $(x,y)$ position.

A classic method applied to DO is Novelty Search (NS)~\citep{lehman2011abandoning}.
NS retains an \emph{archive} of previously found individual solutions.
It aims to expand the archive by finding solutions that are far in the feature space from existing solutions in the archive.
Specifically, NS optimizes for solutions with a high \emph{novelty score}, where novelty score is the average distance in feature space from a solution to its $k$-nearest neighbors in the archive.

While the novelty score is framed in terms of distance, we study its interpretation as an approximation of \emph{density}~\citep{chenu2021selection}, where density is directly proportional to the number of solutions in a region of feature space.
A high novelty score indicates that a solution's features are far from those of its $k$-nearest neighbors in the archive, i.e., it is located in an area of the feature space with low density.

An alternative approach to DO is to apply general-purpose QD algorithms and assume that all solutions in the search space have identical quality, i.e., equal to some constant. Covariance Matrix Adaptation MAP-Annealing (CMA-MAE)~\citep{fontaine2023covariance} is a state-of-the-art QD algorithm that optimizes for archive improvement with the CMA-ES~\citep{hansen:cma16} optimizer.
CMA-MAE retains a discrete archive of solutions in feature space.
When applied to solve a DO problem by ignoring the solution quality, this archive naturally becomes a histogram that represents the distribution of solutions in feature space, with lower values of the histogram indicating lower density.
CMA-MAE then performs \emph{density descent} on this histogram, where it continually seeks to fill the areas of low density.

We emphasize that, to efficiently explore feature space, both NS and CMA-MAE leverage density estimates of solutions in feature space --- NS is guided by novelty score, while CMA-MAE reads from its histogram.
However, both of these density estimates have drawbacks.
On one hand, the novelty score in NS is \emph{continuous} but provides a \emph{weaker} stability guarantee, meaning that its value can change arbitrarily when new features are discovered.
On the other hand, the histogram in CMA-MAE provides a \emph{stronger} stability guarantee but utilizes a \emph{discrete} approximation that gives flat gradient signals on its bins.

\emph{Our key insight is to overcome the drawbacks of current density estimation methods in DO by introducing continuous, stable approximations of the solution density in feature space.} 
This insight results in the following contributions:
(1) We propose density descent search (DDS; \autoref{fig:diagram}), an algorithm that queries continuous density estimates to guide an optimizer, in our case CMA-ES, to search for solutions that are diverse in the feature space (\autoref{sec:dds}).
We propose two variants of DDS: DDS-KDE leverages non-parametric Kernel Density Estimation (KDE)~\citep{chen2017tutorial}, while DDS-CNF learns the underlying density function with continuous normalizing flow~\citep{lipman2022flow}. 
(2) We show that, when combined with a ranking-based optimizer like CMA-ES, NS reduces to a special case of DDS-KDE (\autoref{sec:connection}).
(3) We prove that KDE provides stronger algorithmic stability guarantees than novelty score (\autoref{sec:connection}).
(4) We demonstrate that our DDS algorithms outperform prior work on 3 out of 4 domains (\autoref{sec:experiments}).
(5) We show that our algorithms perform well on multi-dimensional feature spaces, which currently present significant challenges to QD and DO algorithms.

\section{Problem Definition} \label{sec:problem}

\noindent\textbf{Quality Diversity Optimization (QD).}
For solution $\vtheta \in \R^n$, QD assumes an \emph{objective} function $f: \R^n \to \R$ and $m$ \emph{feature} functions, jointly represented as a vector-valued function $\vphi: \R^n \to \R^m$.
We refer to the image of $\vphi$ as the \emph{feature space} $S$.
The goal of QD is to find, for each $\vs \in S$, a solution $\vtheta$ where $f(\vtheta)$ is maximized and $\vphi(\vtheta) = \vs$.

\noindent\textbf{Diversity Optimization (DO).}
DO is a special case of QD where the objective is constant, i.e., $f(\vtheta)=C$. The goal of DO is to find, for each $\vs \in S$, a solution $\vtheta$ where $\vphi(\vtheta) = \vs$.

\section{Background} \label{sec:background}

\subsection{QD and DO Algorithms}
Various methods address the QD and DO problems by relaxing them to the problem of finding an \emph{archive} (i.e., a set) $\mathcal{A}$ of representative solutions. The structure of the archive defines two major families of algorithms: those based on Novelty Search and those based on MAP-Elites.

\noindent\textbf{Novelty Search (NS).}
The key insight of NS~\citep{lehman2011abandoning} is to discover diverse solutions in feature space by optimizing for solutions that are ``novel'' with respect to a current set of solutions.
Given a set of features $\mathcal{X} \subseteq \R^m$, the novelty of a solution $\vtheta \in \R^n$ and its features $\vy = \vphi(\vtheta)$ relative to $\mathcal{X}$ is encapsulated in its \emph{novelty score}, denoted $\rho(\vy; \mathcal{X})$:
\begin{equation}
    \rho(\vy; \mathcal{X}) = \frac{1}{k}\sum_{\vy' \in N_k(\vy; \mathcal{X})} \mathrm{dist}(\vy, \vy') \label{eq:ns}
\end{equation}
where $N_k(\vy; \mathcal{X})$ is the set of $k$-nearest neighbors to $\vy$ in $\mathcal{X}$, and $\mathrm{dist}$ is a distance function --- henceforth, we consider $\mathrm{dist}$ to be some norm $\lVert\cdot\rVert$.

NS maintains an archive $\mathcal{A}$ of unbounded size and generates solutions with an underlying optimizer, traditionally a genetic algorithm.
For each solution $\vtheta$ produced by the optimizer, NS computes the novelty score with respect to the current archive, i.e., $\rho(\vy; \mathcal{A})$.
If $\rho(\vy; \mathcal{A})$ exceeds an \emph{acceptance threshold}, then $(\vtheta, \vy)$ is added to the archive.
In this manner, NS gradually adds novel solutions to the archive and explores the feature space.

Importantly, the novelty score is non-stationary in that the novelty of a solution $\vtheta$ changes as the archive $\mathcal{A}$ is updated.
We prove in \autoref{sec:connection} that the degree of non-stationarity in novelty score is unbounded for general feature spaces, resulting in significant optimization challenges for adaptive optimizers such as CMA-ES.

While we consider general optimization domains in this work, we note that prior works~\citep{nses} have specialized NS for reinforcement learning domains, particularly the case when no feature function is assumed~\citep{paolo2020unsupervised,cully2019autonomous}.

\noindent\textbf{Multi-dimensional Archive of Phenotypic Elites (MAP-Elites).} 
While NS was developed for DO, MAP-Elites~\citep{mouret2015illuminating} was developed for the general QD setting.
Compared to the unstructured archive of NS, MAP-Elites divides the feature space into a predefined tessellation $T: \R^m \to \{1,\dots,l\}$, where $e \in \{1,\dots,l\}$ is a \emph{cell} in the tessellation and $l$ is the total number of cells.
Given a cell $e$, the MAP-Elites archive associates a solution $\vtheta$ with cell $e$ if and only if the solution's features are contained in $e$, i.e., $T(\vphi(\vtheta)) = e$.
Moreover, MAP-Elites stores at most one solution for every cell in the tessellation.

During a QD search, MAP-Elites first draws solutions from a predefined distribution, e.g., a Gaussian, and inserts the solutions into the archive.
Subsequently, MAP-Elites generates and inserts new solutions by mutating existing archive solutions with a genetic operator, such as the Iso+LineDD operator~\citep{vassiliades2018discovering}.
Importantly, when solutions inserted into the archive land in the same tessellation cell, MAP-Elites only retains the solution with the greatest objective value. Thus, MAP-Elites gradually collects high-performing solutions that have diverse features.
We apply MAP-Elites to DO by setting a constant objective $f(\cdot) = C$ over the solution space. 

The choice of tessellation in MAP-Elites can significantly impact scalability.
The most common tessellation is a grid tessellation~\citep{mouret2015illuminating}, which divides the feature space into equally-sized hyperrectangles.
In high-dimensional feature spaces, grid tessellations require exponentially more memory due to the curse of dimensionality.
Thus, a common alternative is the centroidal Voronoi tessellation (CVT)~\citep{vassiliades2018using}, which divides the feature space into $l$ evenly-sized polytopes.

\noindent\textbf{Covariance Matrix Adaptation Evolution Strategy (CMA-ES).}
One recent line of QD algorithms has combined CMA-ES~\citep{hansen:cma16} with MAP-Elites.
CMA-ES is a state-of-the-art single-objective optimizer that represents a population of solutions with a multivariate Gaussian $\mathcal{N}(\vmu, \mSigma)$.
Each iteration, CMA-ES draws $\lambda$ solutions from the Gaussian.
Based on the \textit{rankings} (rather than the raw objectives) of the solutions, CMA-ES adapts the covariance matrix $\mSigma$ to regions of higher-performing solutions.
While CMA-ES is a derivative-free optimizer, it has been shown to approximate a natural gradient~\citep{akimoto2010bidirectional}.

\noindent\textbf{Covariance Matrix Adaptation MAP-Elites (CMA-ME)}.
The first work to integrate CMA-ES with MAP-Elites was CMA-ME~\citep{fontaine2020covariance}.
The key idea of CMA-ME is to optimize for archive improvement with CMA-ES.
Namely, in addition to a MAP-Elites-style archive, CMA-ME maintains an instance of CMA-ES.
The CMA-ES instance samples solutions from a Gaussian, and the solutions are ranked based on how much they improve the archive, e.g., solutions that found new cells in the archive are ranked high, while those that were not added at all are ranked low.
With this \emph{improvement ranking}, CMA-ES adapts the Gaussian to sample solutions that will further improve the archive.
Note that in DO settings, since the objective is a constant value, the improvement ranking places solutions that filled a new cell ahead of solutions that did not add a new cell.

\noindent\textbf{Covariance Matrix Adaptation MAP-Annealing (CMA-MAE).}
One limitation of CMA-ME is that it focuses too much on exploring for new cells in feature space rather than optimizing the objective~\citep{tjanaka2022approximating}.
CMA-MAE~\citep{fontaine2023covariance} addresses this problem by introducing an \emph{archive learning rate} $\alpha$.
When $\alpha = 1$, CMA-MAE maintains the same exploration behavior as CMA-ME, but when $\alpha = 0$, CMA-MAE focuses solely on optimizing the objective, like CMA-ES.
For $0 < \alpha < 1$, CMA-MAE blends between these two extremes, enabling it to both explore feature space and optimize the objective.

In DO settings, where the objective is constant ($f(\cdot) = C$), CMA-MAE naturally performs \emph{density descent}~\citep{fontaine2023covariance}.
Intuitively, the archive becomes a histogram that represents how many solutions have been found in each region of feature space.
Lower values of the histogram indicate lower \emph{density} of solutions, i.e., if a histogram bin has a low value, few solutions have been discovered in that region of feature space.
CMA-MAE then seeks to \emph{descend} the histogram by searching for solutions that fill the low-valued histogram bins.

\subsection{Density Estimation}

We consider two density estimation methods: Kernel Density Estimation and Continuous Normalizing Flows.

\noindent\textbf{Kernel Density Estimation (KDE).}
KDE is a non-parametric density estimation method that does not make any assumptions on the underlying probability density distribution from which samples are drawn~\citep{chen2017tutorial}.
Given a set of features $\mathcal{X} \subseteq \R^m$, bandwidth parameter $h$, and kernel function $K(\cdot)$, KDE computes the density function $\hat{D} : \R^m \to \R$ for a given feature $\vy \in \R^m$ as:
\begin{align}
    \hat{D}_h(\vy; \mathcal{X}) = \frac{1}{|\mathcal{X}|h} \sum_{\vy' \in \mathcal{X}} K\left(\frac{\lVert \vy - \vy'\rVert}{h}\right)
\label{eq:KDE}
\end{align}
Prior work has thoroughly studied the problem of selecting bandwidth that accurately estimate the underlying density function~\citep{silverman1986density, rudemo1982empirical, bowman1984alternative}.
While accurate density estimation is important for KDE, we show in \autoref{thm:kde_stability} that larger bandwidth results in more stable density estimates (albeit at the cost of accuracy), which is beneficial to the proposed algorithms in \autoref{sec:dds}.  

When optimizing a density estimate (e.g., with gradient descent), KDEs have several advantages over histograms.
First, the shape of a histogram depends on its bin size~\citep{weglarczyk2018kernel}.
Second, the binning procedure leads to a discontinuous optimization landscape and flat gradient signals inside each bin~\citep{weglarczyk2018kernel}.
In contrast, KDEs produce a smooth, continuous density function based on the location of the samples in the underlying distribution~\citep{chen2017tutorial}.

\noindent\textbf{Continuous Normalizing Flow (CNF).}
CNF~\citep{lipman2022flow} is a generative modeling method that constructs a diffusion path between a simple distribution (e.g., a Gaussian distribution) and an unknown distribution.
The diffusion path describes a mapping from a point on the simple distribution to a corresponding point on the unknown distribution such that their probability densities are roughly equal. 
CNFs enable sampling from probability distributions where direct sampling is difficult (e.g., distributions of images).
This is accomplished by sampling from the simple distribution and transforming the sample via the learned diffusion path of the CNF.
However, we utilize CNF to estimate the density of feature space for our proposed algorithms in \autoref{sec:dds}.

\section{Density Descent Search} \label{sec:dds}

We present Density Descent Search (DDS), a DO algorithm that efficiently explores the feature space by leveraging continuous density estimates. 
DDS maintains a buffer of features $\gB$ that represents the distribution of features discovered so far.
Based on this buffer, DDS models a density estimation of the feature distribution, and by querying this density estimate, DDS guides an adaptive optimizer to discover solutions in less-dense areas of feature space.
As it searches for such solutions, DDS expands the discovered set of features.
We provide an overview of DDS in \autoref{fig:diagram} and describe the components of DDS below.

\noindent\textbf{Density Estimation.}
The core of DDS is its representation of feature space density.
We propose two variants of DDS that differ in their density representation.
The first variant, DDS-KDE, represents density with KDE.
With the KDE, we compute density via \autoref{eq:KDE}, with $\mathcal{X}$ set to the buffer of features $\gB$.

The second variant, DDS-CNF, estimates the density in feature space by learning a CNF between the standard normal distribution $\mathcal{N}(\bm{0}, \mI)$ and the observed feature space distribution.
Similar to DDS-KDE, the feature space distribution is represented by the buffer $\gB$.
We compute a density estimate at any location in the feature space by integrating an ordinary differential equation~\citep{lipman2022flow}.
Applying techniques from prior work~\citep{lipman2022flow}, we represent the CNF with a neural network, and on each iteration, we finetune the network on the new distribution of features contained in the buffer.

KDE and CNF differ in how easily their smoothness can be controlled.
On one hand, the shape of the KDE can be controlled with the bandwidth hyperparameter.
Higher bandwidth leads to a smoother density estimate but conceals modes of the true density function,  as illustrated in \autoref{fig:bandwidth}.
On the other hand, while CNF does not require selecting a bandwidth hyperparameter, the smoothness of the density estimation cannot be easily controlled, which undermines the performance of the algorithm (\autoref{sec:analysis}).
We discuss the impacts that the smoothness of the density estimate has on DDS in~\autoref{sec:connection}.

\noindent\textbf{Feature Buffer.}
To provide the basis for density estimates in feature space, DDS maintains a buffer $\gB$ that stores the features of sampled solutions (\autoref{dds:buffer_init}).
This buffer represents the \emph{observed} distribution of the feature space and is updated every time new solutions and features are discovered.
In theory, the buffer can have infinite capacity, storing every feature ever encountered.
In practice, the buffer can only retain a finite number of features due to computation and memory limitations.
To decide which features to retain in the buffer, we manage the buffer with an optimal reservoir sampling algorithm~\citep{li1994reservior}.
This algorithm updates the buffer with online samples that accurately represent the distribution of features discovered so far by DDS (\autoref{dds:buffer_update}).

\noindent\textbf{Optimizer.} DDS guides an adaptive optimizer to discover solutions in low-density regions of the feature space.
Examples of such optimizers include xNES~\citep{glasmachers2010exponential} and Adam~\citep{adam}.
We select CMA-ES as the underlying optimizer due to its reputation as a state-of-the-art optimizer~\citep{hansen:cma16} and its high performance in existing QD algorithms~\citep{fontaine2020covariance, fontaine2021differentiable, fontaine2023covariance}.

\noindent\textbf{Passive Archive.} Following prior work~\citep{pugh2015confronting}, to track how much of the feature space has been explored, DDS inserts all discovered solutions and features into a passive MAP-Elites-style (\autoref{sec:background}) archive $\mathcal{A}$.
While DDS itself only uses the archive to record solutions and features, the archive is useful when computing metrics in experiments (\autoref{sec:design}).

\noindent\textbf{Summary.}
\autoref{alg:dds} shows how the components of DDS come together.
DDS begins by initializing its various components (\autoref{dds:cmaes_init}-\ref{dds:buffer_init}).
During the main loop (\autoref{dds:mainloop}), DDS samples solutions with the underlying optimizer (\autoref{dds:sample}). Since our optimizer is CMA-ES, sampling consists of drawing from a multivariate Gaussian $\mathcal{N}(\vmu, \mSigma)$.
Subsequently, DDS computes the features of the solutions (\autoref{dds:evaluate}) and estimates their feature space density (\autoref{dds:calc_density}).
After sampling solutions, DDS updates its various components.
For instance, in DDS-KDE, the density update (\autoref{dds:density_update}) consists of replacing the feature buffer, and in DDS-CNF, the update involves fine-tuning the neural network on the feature buffer.
Furthermore, solutions with lower density are ranked first on \autoref{dds:ranking}, causing CMA-ES to adapt (\autoref{dds:adapt}) towards less-dense regions of the feature space.

\begin{algorithm}[t]
\caption{Density Descent Search (DDS)}
\label{alg:dds}
\SetKwProg{DDS}{DDS}{}{}
\DontPrintSemicolon
\DDS{($\bm{\phi}(\cdot)$, $density$, $b, N, \lambda, \vmu_0, \sigma$)} {
\KwIn{Feature function $\bm{\phi}(\cdot)$, density estimator object $density$, buffer size $b$, number of iterations $N$, batch size $\lambda$, initial solution $\vmu_0$, and initial step size $\sigma$}
\KwResult{Generates $N \cdot \lambda$ solutions, storing a representative subset of them in a passive archive $\mathcal{A}$.}

\BlankLine
Initialize CMA-ES search point $\vmu := \vmu_0$, search direction $\mSigma=\sigma \mI$, and internal parameters $\vp$ \label{dds:cmaes_init}

Initialize empty archive $\mathcal{A}$ \label{dds:archive_init}

Initialize empty buffer $\gB$ of size $b$ \label{dds:buffer_init}

\For{$iter\gets 1$ \KwTo $N$}{ \label{dds:mainloop}
    Update buffer $\gB$ with $\vphi_{1..\lambda}$ via reservoir sampling \label{dds:buffer_update}
    
    \For{$i\gets 1$ \KwTo $\lambda$}{
        $\vtheta_i \sim \mathcal{N}(\vmu,\mSigma)$ \label{dds:sample} 
        
        $\vphi_i \gets \vphi(\vtheta_i)$ \label{dds:evaluate}
    
        $D_i \gets density(\vphi_i)$ \label{dds:calc_density}

        Add $(\vtheta_i, \vphi_i)$ to archive $\mathcal{A}$ \label{dds:archive_addition}
    }

    Update $density$ with the new buffer $\gB$ \label{dds:density_update}
    
    Rank $\vtheta_i$ in \emph{ascending} order by $D_i$ \label{dds:ranking}
    
    Adapt CMA-ES parameters $\vmu,\mSigma,\vp$ based on density ranking $D_i$ \label{dds:adapt}

}
}
\end{algorithm}

\section{Connection between Novelty Search and Kernel Density Estimates} \label{sec:connection}
We provide theoretical insight into the connection between NS and DDS-KDE and delineate the advantage of KDE over novelty score.
KDE and novelty score are non-stationary since they change as more solutions are discovered by their respective optimizers.
However, the magnitude of non-stationarity is potentially different in KDE and novelty score (\autoref{thm:kde_stability}, \ref{thm:ns_stability}).
Furthermore, when $k \ge |\gB|$ in the $k$-nearest neighbor calculation for novelty score, novelty search becomes a special case of DDS-KDE (\autoref{thm:ns=kde}).
Finally, we conjecture that any meaningful density estimator utilizing a point's distance to its $k$-nearest neighbors will incur similar weak stability guarantees as novelty score (\renewcommand{\theoremautorefname}{Conjecture}\autoref{conj:instability}\renewcommand{\theoremautorefname}{Theorem}).
We include all proofs in \autoref{sec:proofs}.

\noindent\textbf{Stability Under Non-stationarity.}
DO algorithms such as NS, DDS, and CMA-MAE gradually learn a non-stationary density representation (e.g., KDE or histogram) as they explore the feature space.
However, drastic changes in the representation present significant challenges for adaptive optimizers, such as Adam~\citep{adam} and CMA-ES.
If the density representation changes drastically every iteration, it would be impossible for adaptive optimizers to properly adapt its parameters to optimize the density function.

To characterize the extent of change in the density estimate, we appeal to the notion of \emph{uniform stability}~\citep{hardt2016train}, defined as follows:
\begin{definition}
    Let a function $F(\vx;\gB): \R^d \to \R$ be parameterized by a set $\gB \subseteq \R^d$.
    We say that $F$ is $\epsilon$-uniformly stable if for all $\gB, \gB' \subseteq \R^d$, where $\gB$ and $\gB'$ differ by at most one element, we have
    \begin{equation}
        \sup_{\vx \in \R^d} \left|F(\vx; \gB) - F(\vx; \gB')\right| \le \epsilon \label{eq:uniform_stability}
    \end{equation}
\end{definition}
We note that uniform stability has close connections to influence functions in statistics, which measure how much an estimator deviates from the ground truth when given a subset of the data~\citep{cook1980characterizations}.
In this interpretation, uniform stability is an upper-bound of the (empirical) influence function.

We prove that, for KDE, the higher the size of its feature buffer and bandwidth, the stronger the uniform stability guarantee.

\begin{theorem}{\label{thm:kde_stability}}
    A kernel density estimate $\hat{D}_h(\vx;\gB)$ managed with reservoir sampling, such that features in the buffer are exchanged one at a time, is $\frac{1}{|\gB|h}$-uniformly stable, where $h$ is the bandwidth.
\end{theorem}

Higher bandwidth makes the function more stationary and thus more suited for adaptive optimizers.
However, higher bandwidth also leads to \emph{over-smoothing} of the density estimate, which conceals modes of the true density function~\citep{chen2017tutorial} (\autoref{fig:bandwidth}).
Thus, selecting an optimal bandwidth for DDS requires a fine balance between the accuracy and uniform stability of the KDE.

In contrast, novelty score becomes less uniformly stable as the diameter of the feature space $W$ increases:

\begin{theorem}{\label{thm:ns_stability}}
    Novelty score $\rho(\vx;\gB)$ is $\frac{W}{k}$-uniformly stable, where $k$ is the nearest neighbors parameter in novelty score and $W = \max_{\vs_1,\vs_2 \in S} \lVert \vs_1 - \vs_2 \rVert$ is the diameter of the feature space $S$.
\end{theorem}

Therefore, for unbounded feature spaces, the uniform stability of novelty score is also unbounded, and for bounded feature spaces, novelty score has stronger uniform stability guarantees when $k$ is larger.

When the feature space is bounded, as in our experiments, KDE has a stronger stability guarantee than novelty score for bandwidth $h \ge \frac{k}{|\gB|}$.
Following this insight, we select a bandwidth satisfying this inequality with the empirical experiments in \autoref{sec:bandwidth}.
Our theoretical results are corroborated by our experiments in \autoref{sec:experiments}, which demonstrate that DDS-KDE outperforms NS in all domains on all metrics.

\noindent\textbf{Equivalence of NS and DDS-KDE.}
We observe that, when $h = 1$ and as $k \to |\gB|$, the uniform stability upper-bound of novelty score approaches that of KDE.
We prove in \autoref{thm:ns=kde} that when all points are considered in the novelty score computation (i.e., let $k = \infty$), NS is a special case of DDS-KDE under ranking-based optimizers such as CMA-ES.
Specifically, we demonstrate that under these conditions, the ranking of solutions based on their novelty score is identical to the ranking based on their kernel density estimate.

\begin{theorem}{\label{thm:ns=kde}}
Let $\gB \subseteq \R^m$ be a set of features.
Consider the rankings $\pi_\mathrm{NS}$ and $\pi_\mathrm{KDE}$ on another set of features $\{\vphi_1, \dots, \vphi_m\} \subseteq \R^m$ where $\vphi_i = \vphi(\vtheta_i)$.
We define the rankings as follows: $\pi_\mathrm{NS}(i) \ge \pi_\mathrm{NS}(j) \iff \rho(\vphi_i; \gB) \ge \rho(\vphi_j; \gB)$ and $\pi_\mathrm{KDE}(i) \ge \pi_\mathrm{KDE}(j) \iff \hat{D}(\vphi_i; \gB) \le \hat{D}(\vphi_j; \gB)$.
We show that $\pi_{NS} = \pi_{KDE}$, when NS has $k=\infty$ (or, equivalently, $k = |\gB|$) and KDE has triangular kernel $K(u) = 1 - |u|$ with support over the entire feature space $S$.
\end{theorem}

\noindent\textbf{Stability of $\vk$-NN.}
We conjecture that only considering $k$-nearest neighbors in the density representation, for $k < |\gB|$, will inevitably result in poor uniform stability bounds.
Notice that the maximum distance between any two features is the diameter of the feature space, $W$.
Thus, when a new feature replaces a former $k$-nearest neighbors of $\vx$, the total distance between a feature $\vx$ to its $k$-nearest neighbors can change by $W$ in the worst case.
\begin{conjecture}{\label{conj:instability}}
Let $\mathcal{X} \subseteq \R^m$ and $D: \R^m \to \R$ be a density estimator.
If $D(\vx; \mathcal{X})$ for $\vx \in \R^m$ is computed based on the distance to its $k$-nearest neighbors of $\vx$, then $D(\vx; \mathcal{X})$ is $O\left(W\right)$-uniformly stable for $k < |\gB|$.
\end{conjecture}
A proof of this conjecture would imply that density estimators based on $k$-nearest neighbor metric have unbounded uniform stability in unbounded feature spaces, for $k < |\gB|$.
Thus, the potential for rapid changes in the optimization landscape makes such estimators incompatible with adaptive optimizers.

\section{Experiments} \label{sec:experiments}
We compare the performance of DDS-KDE and DDS-CNF with the QD algorithms MAP-Elites (line), CMA-ME, and CMA-MAE, and with NS using CMA-ES as the underlying optimizer.
MAP-Elites (line) refers to the implementation of MAP-Elites with the Iso+LineDD operator~\citep{vassiliades2018discovering}.
All algorithms are implemented with pyribs \citep{pyribs}.

Our experiments include canonical benchmark domains from QD and DO: Linear Projection~\citep{fontaine2020covariance}, Arm Repertoire~\citep{vassiliades2018discovering}, and Deceptive Maze~\citep{lehman2011abandoning}.
As discussed in~\autoref{sec:problem}, we convert QD domains into DO domains by setting the objective to be constant, effectively removing the importance of solution quality.
For Deceptive Maze, we use the implementation in Kheperax~\citep{grillotti2023kheperax}.
Furthermore, to address the challenges posed by high-dimensional feature spaces, we introduce a new domain, Multi-feature Linear Projection, that generalizes Linear Projection to high-dimensional feature spaces. 

All experiments run on a 128-core workstation with an NVIDIA RTX A6000 GPU. While all algorithms are single-threaded, we use the GPU for training the CNF in DDS-CNF and for evaluating the Deceptive Maze domain.

\subsection{Experimental Design}\label{sec:design}

\noindent\textbf{Independent Variable.} In each domain, we conduct a between-groups study with the algorithm as the independent variable.

\noindent\textbf{Dependent Variables.} We compute the archive coverage as the number of occupied cells in the archive divided by the total number of cells.
To compare the coverage of tessellation-based algorithms (CMA-MAE, CMA-ME, and MAP-Elites) with that of non-tessellation-based algorithms (DDS-KDE, DDS-CNF, and NS), we track the coverage of all algorithms on a passive archive $\mathcal{A}$ tessellated by a $100 \times 100$ grid. %
For the high-dimensional feature space experiment, we track coverage with a centroidal Voronoi tessellation with 10,000 cells~\citep{vassiliades2018using}.

Following prior work~\citep{gomes15devising}, we also assess the ability of each algorithm to uniformly explore the feature space. 
Thus, we measure the cross-entropy between a uniform distribution and the distribution of sampled features.
Let $N_e$ be the total times throughout the entire search that solutions were discovered in cell $e \in \{1, \dots, l\}$ in the passive archive, and $N_\mathrm{total} = \sum_{e=1}^l N_e$.
The cross-entropy score is defined as:
\begin{equation}
    \mathrm{CE} = -\sum_{e=1}^l \frac{1}{l}\log\left(\frac{N_e}{N_\mathrm{total}}\right)
\end{equation}
$\mathrm{CE}$ achieves its minimum value when $N_e$ is uniformly distributed across all cells $e \in \{1, \dots, l\}$. For all passive archives in our experiments, this minimum is 9.21 for $l = \num[group-separator={,}]{10000}$ cells.

\subsection{Domain Details} \label{sec:domain}

\textbf{Linear Projection (LP).} LP is a QD domain where $\vphi$ induces high distortion by mapping an $n$-dimensional solution space to a 2D feature space~\citep{fontaine2020covariance}.
A harsh penalty is applied outside the bounds of the feature space to hinder exploration near the bounds.
The feature function $\vphi: \R^n \to \left[-5.12 \cdot \frac{n}{2}, 5.12 \cdot \frac{n}{2}\right]^2$ is defined as:
\begin{align}
    \vphi(\vtheta) &= \left(\sum_{i=1}^{\frac{n}{2}}\mathrm{clip}(\theta_i), \sum_{i=\frac{n}{2} + 1}^{n}\mathrm{clip}(\theta_i)\right) \label{eq:linproj2d} \\
    \mathrm{clip}(\theta_i) &= 
    \begin{cases}
        \theta_i &\text{if}\; |\theta_i| \le 5.12\\
        5.12/\theta_i & \text{otherwise} \\
    \end{cases}
\end{align}
where $\theta_i$ is the $i$th component of $\vtheta$, and we assume that $n$ is divisible by $2$.
$\vphi(\vtheta)$ applies $\mathrm{clip}(\cdot)$ to each $\theta_i$ and sums the two halves of $\vtheta$.
Since $\mathrm{clip}(\theta_i)$ restricts $\theta_i$ to the interval $[-5.12, 5.12]$, $\vphi(\vtheta)$ is bounded by the closed interval $\left[-5.12 \cdot \frac{n}{2}, 5.12 \cdot \frac{n}{2}\right]^2$.

For DO experiments, we set the objective function $f(\bm{\theta}) = 1$.
Following prior work~\citep{fontaine2021differentiable, fontaine2023covariance}, we let the solution dimension be $n = 100$.

\noindent\textbf{Arm Repertoire.} The goal of Arm Repertoire~\citep{cully2017quality, vassiliades2018discovering} is to search for a diverse collection of arm positions for a planar robotic arm with $n$ revolving joints.
In this domain, $\vtheta \in [-\pi, \pi]^n$ represents the angles of the $n$ joints, and $\vphi(\vtheta)$ computes the $(x, y)$ position of the arm's end-effector using forward kinematics.
While all other domains in this paper have a maximum of $100\%$ coverage, Arm Repertoire has a maximum coverage of $80.24\%$ when using a $100 \times 100$ grid archive~\citep{fontaine2023covariance}, since the arm can only move in a circle of radius $n$.
Similar to LP, we set $f(\vtheta) = 1$ and $n = 100$.

\begin{figure}
\centering
\begin{subfigure}{0.48\linewidth}
\includegraphics[width=\linewidth]{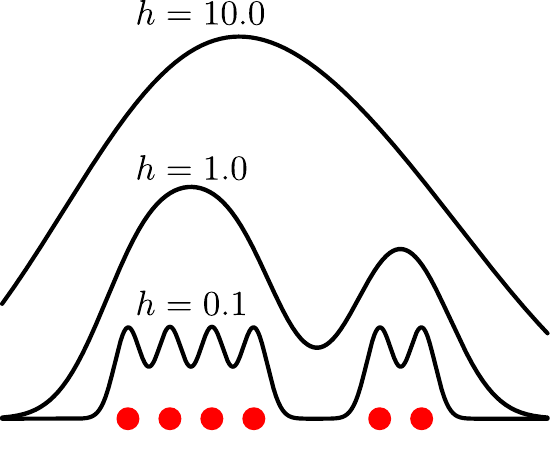}
\caption{The effect of the bandwidth $h$ on a one-dimensional KDE. Red dots represent the data, and black lines depict the KDE. 
When $h$ is too small, the KDE reveals many misleading local maxima.
When $h$ is too large, the KDE conceals modes from the underlying distribution.}
\label{fig:bandwidth}
\end{subfigure}
\hfill
\begin{subfigure}{0.48\linewidth}
\centering
\includegraphics[width=0.9\linewidth]{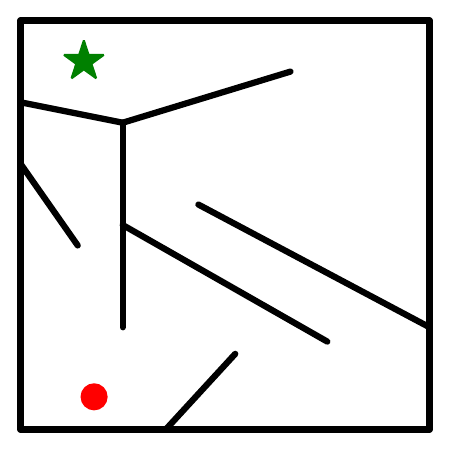}
\caption{The layout in the deceptive maze domain. The red dot indicates the start, and the green dot indicates the goal. While there is a goal, DO seeks to find agents that reach every position in the maze.
Adapted from prior work~\citep{lehman2011abandoning}. 
}
\label{fig:maze}
\end{subfigure}
\caption{}
\label{fig:mazeandbandwidth}
\end{figure}

\noindent\textbf{Deceptive Maze.} Deceptive Maze is a DO domain that challenges the algorithm to discover a diverse set of final positions for robots navigating a maze (\autoref{fig:maze})~\citep{lehman2011abandoning}.
In this domain, $\vtheta$ parameterizes the robot's neural network controller.
$\vphi(\vtheta)$ is the final position of the robot after evaluating its path in the maze.
As a DO domain, this has no objective function.
In our experiments, the neural network is a MLP with $n=66$ parameters.

\noindent\textbf{Multi-feature Linear Projection (Multi-feature LP).}
We introduce a generalized version of LP that scales to $m$-dimensional feature spaces.
Assuming that $n$ is divisible by $m$, let $r = \frac{n}{m}$.
The feature function $\vphi: \R^n \to \left[-5.12 \cdot r, 5.12 \cdot r \right]^m$ is defined as
\begin{equation}
    \vphi(\vtheta) = \left(\sum_{i=jr + 1}^{(j+1)r}\mathrm{clip}(\theta_i) : j \in \{0, \dots, m-1\}\right) \label{eq:multifeature}
\end{equation}
When $m=2$, this is equivalent to LP definition in \autoref{eq:linproj2d}.
Our experiments use $n=100$ and $m=10$.

\subsection{Analysis}\label{sec:analysis}

\begin{figure*}
  \includegraphics[width=\linewidth]{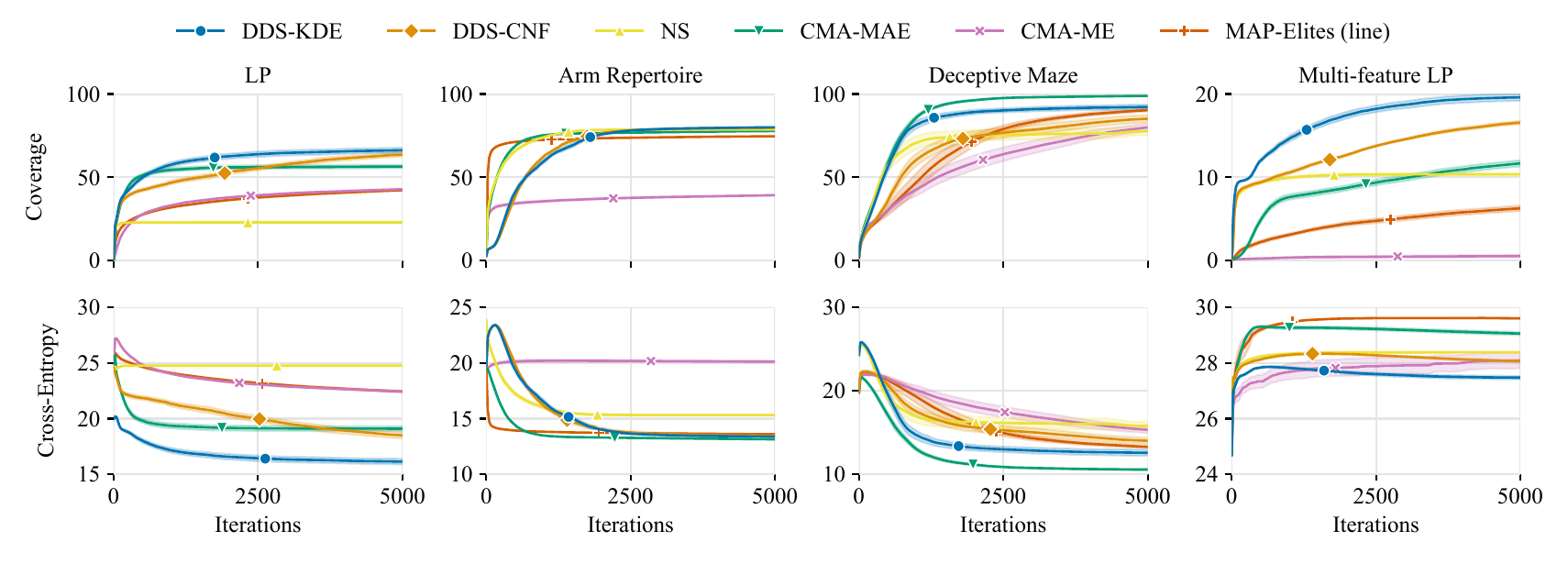}
  \centering
  \setlength{\tabcolsep}{3.6pt}
  \fontsize{7.8}{7.1}\selectfont
  \begin{tabular}{lrrrrrrrr}
  \toprule
    & \multicolumn{2}{c}{LP} & \multicolumn{2}{c}{Arm Repertoire} & \multicolumn{2}{c}{Deceptive Maze} & \multicolumn{2}{c}{Multi-feature LP} \\
  \cmidrule(r){2-3}
  \cmidrule(r){4-5}
  \cmidrule(r){6-7}
  \cmidrule(r){8-9}
   & Coverage & Cross-Entropy & Coverage & Cross-Entropy & Coverage & Cross-Entropy & Coverage & Cross-Entropy \\
  \midrule
  DDS-KDE & {\bf67.67 \textpm 2.13\%} & {\bf17.57 \textpm 0.41} & {\bf80.22 \textpm 0.01\%} & 14.14 \textpm 0.02 & 94.93 \textpm 0.84\% & 11.84 \textpm 0.21 & {\bf50.22 \textpm 0.45\%} & {\bf22.20 \textpm 0.09} \\
  DDS-CNF & 63.65 \textpm 1.38\% & 18.50 \textpm 0.29 & 79.82 \textpm 0.19\% & 13.48 \textpm 0.12 & 85.17 \textpm 2.28\% & 14.01 \textpm 0.38 & 16.57 \textpm 0.22\% & 28.07 \textpm 0.03 \\
  NS & 22.96 \textpm 1.59\% & 24.78 \textpm 0.06 & 78.67 \textpm 0.23\% & 15.32 \textpm 0.07 & 77.70 \textpm 2.47\% & 15.78 \textpm 0.50 & 10.36 \textpm 0.16\% & 28.39 \textpm 0.03 \\
  CMA-MAE & 56.47 \textpm 1.06\% & 19.09 \textpm 0.23 & 77.77 \textpm 0.10\% & {\bf13.14 \textpm 0.01} & {\bf98.83 \textpm 0.18\%} & {\bf10.55 \textpm 0.06} & 11.68 \textpm 0.37\% & 29.06 \textpm 0.06 \\
  CMA-ME & 42.90 \textpm 0.25\% & 22.45 \textpm 0.05 & 39.29 \textpm 0.24\% & 20.13 \textpm 0.09 & 79.97 \textpm 2.17\% & 15.33 \textpm 0.39 & 0.56 \textpm 0.14\% & 28.10 \textpm 0.28 \\
  MAP-Elites (line) & 42.30 \textpm 0.31\% & 22.43 \textpm 0.07 & 74.62 \textpm 0.07\% & 13.60 \textpm 0.01 & 90.32 \textpm 0.87\% & 13.28 \textpm 0.16 & 6.28 \textpm 0.34\% & 29.61 \textpm 0.02 \\
  \bottomrule
  \end{tabular}

  \caption{Coverage and cross-entropy (CE) after 5,000 iterations of each algorithm in all domains. We report the mean over 10 trials, with error bars showing the standard error of the mean. Higher coverage and lower cross-entropy are better.}
  \label{fig:results}
\end{figure*}

\autoref{fig:results} shows the mean coverage and cross-entropy over 10 trials.
In each domain, we conducted a one-way ANOVA to check if the algorithms differed in their coverage and cross-entropy.
Since all ANOVAs were significant ($p < 0.001$), we followed up with pairwise comparisons via Tukey's HSD test.
We include full statistical analysis in \autoref{sec:stats}.

\noindent\textbf{Coverage.}
DDS-KDE and DDS-CNF outperform all baselines on LP, Arm Repertoire, and Multi-feature LP.
They exhibit no statistical difference in performance with the best-performing algorithm on Deceptive Maze (CMA-MAE).
Notably, DDS-KDE solves Arm Repertoire nearly perfectly, as the maximum coverage in the domain is $80.24\%$.

We attribute the high coverage of DDS-KDE and DDS-CNF on these domains to the continuity of their density estimate, which prevents DDS-KDE from converging prematurely.
For example, on LP, our algorithms discover more solutions near the edges of the feature space than CMA-MAE (\autoref{sec:heatmaps}).
The passive archive maintained by CMA-MAE converges as all the solutions fall into previously discovered cells~\citep{fontaine2020covariance}.
In contrast, the continuity of our density estimate and the online buffer updates facilitate DDS algorithms to achieve slow but continual progress in exploring the feature space (\autoref{fig:results}, LP and Arm Repertoire).
This is because the shape of the continuous density estimate always changes slightly after each iteration of the algorithm.

For multi-feature LP, we observe that algorithms leveraging continuous representations of the feature space, i.e., DDS and NS, explore the feature space much faster than other algorithms driven by discrete feature space representations, e.g., CMA-MAE (\autoref{fig:results}).
This is because CMA-MAE is optimizing on a centroidal Voronoi tessellation with 10,000 cells, where each cell has significantly more volume compared to that of a $100 \times 100 $ grid due to the increased dimensionality of the feature space.
Hence, more solutions map to the same cells, making it more difficult to find solutions outside of previously explored cells.

A similar phenomenon was observed in prior work when increasing the dimensionality of the solution space in the LP domain~\citep{fontaine2023covariance}.
Increased solution dimensionality more heavily distorts the feature mapping and, similarly, causes most solutions to fall in the same cells of the feature space.
Consequently, the LP domain becomes exceptionally challenging for QD algorithms working with a tessellation (like CMA-MAE), as there is insufficient signal to drive the algorithm towards the boundaries of the feature space.

DDS-KDE circumvents this limitation of tessellations with its continuous density estimation of the feature space.
While in tesselation-based algorithms, solutions will fall into the same cell, DDS-KDE will retain the solutions in its buffer, which generates signal to drive the search towards the feature space boundaries.
Thus, DDS-KDE is more resilient to the distortions of the feature mapping and can better scale with the dimensionality of the feature space.

\noindent\textbf{Cross-Entropy.}
DDS-KDE and DDS-CNF outperform\footnote{Recall that lower CE is better as it indicates a more uniform exploration of the feature space (\autoref{sec:design})} all baselines in LP, with the exception of DDS-CNF, whose performance is not significantly different from CMA-MAE.
On Arm Repertoire and Deceptive Maze, DDS-KDE and DDS-CNF outperform NS.
However, while DDS-KDE is on par with CMA-MAE on Arm Repertoire , CMA-MAE outperforms DDS-KDE on Deceptive Maze and DDS-CNF on both Arm Repertoire and Deceptive Maze.
Finally, DDS-KDE outperforms all algorithms on Multi-feature LP, while DDS-CNF outperforms CMA-MAE and MAP-Elites (line).

We attribute the improved performance of CMA-MAE to the nature of the cross-entropy metric.
Cross-entropy is intended to approximate the distributional similarity between the exploration of the feature space and the uniform distribution.
CMA-MAE directly optimizes a passive archive with uniform tessellation, unlike DDS and NS, which are agnostic to the passive archive.
This experimental setup naturally favors CMA-MAE with respect to the cross-entropy metric.

DDS-KDE performs as well as DDS-CNF in terms of coverage and cross-entropy across most domains.
However, on Multi-feature LP, DDS-KDE outperforms DDS-CNF on both metrics; on LP and Deceptive Maze, DDS-KDE outperforms DDS-CNF in terms of cross-entropy.

We attribute the difference between the performance of DDS-KDE and DDS-CNF to the bandwidth parameter in KDE, which allows us to adjust the smoothness of the KDE (\autoref{fig:bandwidth}). On the other hand, CNF lacks explicit control over its smoothness. 
Thus, while we adjust the smoothness of the KDE to improve performance for DDS-KDE (\autoref{sec:connection}), we can not apply the same techniques to boost the performance of DDS-CNF.

\section{Limitations and Future Work} \label{sec:future}

We introduce a new diversity optimization algorithm with two variants, DDS-KDE and DDS-CNF.
We provide theoretical insight into the connection between DDS-KDE and novelty search and show that both DDS algorithms excel at discovering diverse solutions.

We envision several extensions to DDS.
While we proposed two DDS variants, DDS is a general method that can integrate with a wide variety of distribution fitting techniques.
DDS can be combined with parametric estimators like mixture models~\citep{mclachlan2019finite} and moment matching~\citep{han2018local}, as well as with non-parametric estimators like vector quantization~\citep{kohonen1990improved} and generative models~\citep{goodfellow2020generative}.
In general, we believe that DDS will interface with any method that learns a continuous density representation of the feature space.

For DDS-KDE, while we assumed a constant bandwidth, prior work has adapted the bandwidth online~\citep{kristan2011multivariate} and varied the bandwidth based on query location~\citep{terrell1992variable}.
These methods could be applied to secure tighter uniform stability bounds for the KDE, making feature space exploration more efficient.

A key limitation of our algorithms is their significant computational cost.
Querying the KDE incurs a runtime of $O(|\gB|)$, where $|\gB|$ is the buffer size, and fine-tuning the CNF is computationally expensive due to backpropagation, even when accelerated on a GPU.
These issues could be mitigated by accelerating KDE computation~\citep{obrien2016fast, yang2003improved} and improving CNF training efficiency~\citep{onken2021ot, huang2021accelerating}.

We proposed a generalization of the standard linear projection domain to higher dimensional feature spaces, and demonstrate that DDS exhibits superior performance to existing algorithms.
Our results could be further strengthened by evaluating DDS's performance on more complex domains.
For example, we could evaluate the performance of DDS on locomotion tasks, where the solutions are neural network controllers and the task is to control a multi-pedal walker to move efficiently~\citep{Sharma2020Dynamics-Aware, rakicevic2021policy}.

Finally, we are excited to see how the underlying insight of DDS --- leveraging continuous density representations of the feature space --- can improve the exploration power of QD algorithms, especially in high-dimensional feature spaces.
Applying DDS to QD will improve its performance on applications such as scenario generation for complex systems~\citep{fontaine2021importance}, generative design~\citep{gaier2018data}, damage recovery in robotics~\citep{cully2015robots}, reinforcement learning~\citep{batra2023proximal}, and procedural content generation~\citep{gravina2019procedural}.

\begin{acks}
The authors would like to thank the anonymous referees for
their valuable comments and helpful suggestions.
This work was supported by the \grantnum{NSF}{NSF NRI (\#2024949)}, \grantnum{NSF}{NSF GRFP (\#DGE-1842487)}, \grantnum{NSF}{NSF CAREER (\#2145077)}, and the NVIDIA Academic Hardware Grant.

\end{acks}

\bibliographystyle{ACM-Reference-Format}
\bibliography{references}

\appendix

\onecolumn
\section{Hyperparameters} \label{sec:hyperparameters}
We provide an overview of the hyperparameters of our experiments.
\autoref{tab:hyperparameters} lists hyperparameters that varies across algorithms and domains.
Hyperparameters that are fixed for the algorithms are described below. 

Following prior work~\citep{fontaine2020covariance}, all algorithms run with 15 emitters --- emitters~\citep{pyribs} are objects that generate solutions for QD algorithms.
For MAP-Elites (line), the emitter applies the Iso+LineDD operator to random solutions from the archive~\citep{vassiliades2018discovering}.
For other algorithms, the emitters are individual instances of CMA-ES.
In each experiment, we run the algorithm for 5000 iterations, where each emitter generates $\lambda = 36$ solutions per iteration.
Following prior work~\citep{fontaine2023covariance}, CMA-ES is configured with the $\mu$ selection rule, \texttt{basic} restart rule, and an initial solution $\vx_0 = \vzero$.

For LP, Arm Repertoire, and Deceptive Maze, MAP-Elites-based algorithms (CMA-MAE, CMA-ME, and MAP-Elites (line)) run on a grid archive of $100 \times 100$ following~\citep{fontaine2023covariance}.
For Multi-feature LP, since maintaining a grid tessellation would require an exponential number of cells in the number of dimensions, we use a CVT archive with 10,000 cells.
Due to randomness in the generation of the CVT cells, we pre-generate the centroids and use the same centroids for all algorithms. 

Our DDS algorithms (DDS-KDE and DDS-CNF) maintain a buffer of capacity $b = \num[group-separator={,}]{10000}$.
DDS-KDE uses a Gaussian kernel with a bandwidth parameter of $h = 25.6, 10, 0.01, 2.56$ for Linear Projection, Arm Repertoire, Deceptive Maze, and Multi-feature LP, respectively.
The bandwidth is found through empirical fine-tuning (\autoref{sec:bandwidth}).

\section{Bandwidth Selection for DDS-KDE} \label{sec:bandwidth}
The bandwidth parameter $h$ in KDE controls the smoothness of the estimation (\autoref{fig:bandwidth}).
While prior work~\citep{silverman1986density, rudemo1982empirical, bowman1984alternative} has developed bandwidth selection methods for accurate density estimations, our theoretical results show that the optimal bandwidth for DDS-KDE may require an over-smoothing of the KDE (\autoref{sec:connection}).

To examine the effect of the bandwidth parameter on the performance of DDS-KDE, we perform ablation experiments on LP, Arm Repertoire, and Deceptive Maze.
While applying the same hyperparameters as described in \autoref{sec:hyperparameters}, we vary the bandwidth $h \in (0, W]$, where $W$ is the bound of the feature space defined for each domain (\autoref{sec:domain}).
For instance, in the LP domain, $W_\mathrm{LP} = 5.12 \cdot \frac{n}{2}$ where $n$ is the dimensionality of the solution space.
For conciseness and consistency, we report the results in terms of the normalized bandwidth $h_0 = \frac{h}{W_\mathrm{domain}}$ such that $h_0 \in (0,1]$. $W_\mathrm{domain}$ denotes the feature space bound for the particular domain.

We observe a performance peak across all domains around $h_0 = 0.05$ (\autoref{fig:bandwidth_ablation}).
For LP, Deceptive Maze, and Multi-feature LP, DDS-KDE's performance exhibits diminishes as the bandwidth increases for $h_0 \ge 0.05$.
For Arm Repertoire, DDS-KDE achieves near-optimal coverage for all $h_0 \ge 0.05$.

\begin{figure*}
  \includegraphics[width=\linewidth]{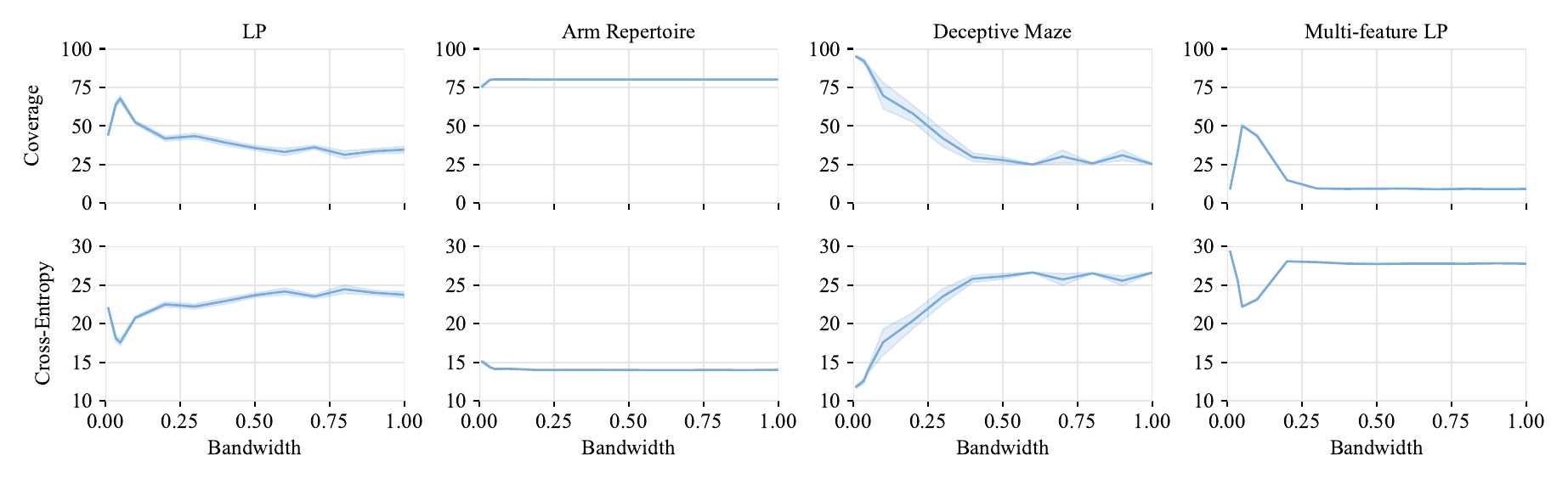}
  \centering
  \setlength{\tabcolsep}{3.6pt}
  \fontsize{7.8}{7.1}\selectfont
  \begin{tabular}{lrrrrrrrr}
  \toprule
   & \multicolumn{2}{c}{LP} & \multicolumn{2}{c}{Arm Repertoire} & \multicolumn{2}{c}{Deceptive Maze} & \multicolumn{2}{c}{Multi-feature LP} \\
   \cmidrule(r){2-3}
   \cmidrule(r){4-5}
   \cmidrule(r){6-7}
   \cmidrule(r){8-9}
   & Coverage & Cross-Entropy & Coverage & Cross-Entropy & Coverage & Cross-Entropy & Coverage & Cross-Entropy \\
  \midrule
  $h_0 = 0.01$ & 44.57 \textpm 0.75\% & 21.99 \textpm 0.14 & 75.75 \textpm 1.06\% & 15.07 \textpm 0.22 & \hl{\bf94.93 \textpm 0.84\%} & \hl{\bf11.84 \textpm 0.21} & 9.50 \textpm 0.15\% & 29.33 \textpm 0.02 \\
  $h_0 = 0.035$ & 64.00 \textpm 1.82\% & 18.10 \textpm 0.37 & 79.84 \textpm 0.20\% & 14.39 \textpm 0.09 & 92.22 \textpm 1.44\% & 12.56 \textpm 0.37 & 33.17 \textpm 0.26\% & 25.57 \textpm 0.05 \\
  $h_0 = 0.05$ & \hl{\bf67.67 \textpm 2.13\%} & \hl{\bf17.57 \textpm 0.41} & \hl{\bf80.22 \textpm 0.01\%} & \hl{14.14 \textpm 0.02} & 87.70 \textpm 1.60\% & 13.98 \textpm 0.35 & \hl{\bf50.22 \textpm 0.45\%} & \hl{\bf22.20 \textpm 0.09} \\
  $h_0 = 0.1$ & 52.41 \textpm 1.18\% & 20.75 \textpm 0.22 & 80.20 \textpm 0.01\% & 14.14 \textpm 0.04 & 69.63 \textpm 8.59\% & 17.60 \textpm 1.69 & 43.61 \textpm 0.81\% & 23.13 \textpm 0.15 \\
  $h_0 = 0.2$ & 41.90 \textpm 1.69\% & 22.50 \textpm 0.33 & 80.15 \textpm 0.00\% & 13.99 \textpm 0.02 & 57.95 \textpm 5.43\% & 20.41 \textpm 1.03 & 14.88 \textpm 0.18\% & 28.08 \textpm 0.02 \\
  $h_0 = 0.3$ & 43.48 \textpm 1.79\% & 22.22 \textpm 0.32 & 80.10 \textpm 0.01\% & 14.01 \textpm 0.02 & 41.98 \textpm 5.51\% & 23.54 \textpm 0.98 & 9.53 \textpm 0.12\% & 27.97 \textpm 0.03 \\
  $h_0 = 0.4$ & 39.37 \textpm 2.00\% & 22.92 \textpm 0.40 & 80.10 \textpm 0.01\% & 14.00 \textpm 0.01 & 29.77 \textpm 2.80\% & 25.81 \textpm 0.47 & 9.24 \textpm 0.09\% & 27.78 \textpm 0.03 \\
  $h_0 = 0.5$ & 35.74 \textpm 1.49\% & 23.68 \textpm 0.28 & 80.09 \textpm 0.00\% & 13.99 \textpm 0.02 & 27.90 \textpm 2.12\% & 26.14 \textpm 0.37 & 9.39 \textpm 0.19\% & 27.74 \textpm 0.04 \\
  $h_0 = 0.6$ & 33.19 \textpm 2.39\% & 24.17 \textpm 0.44 & 80.10 \textpm 0.01\% & 13.99 \textpm 0.01 & 25.04 \textpm 0.24\% & 26.63 \textpm 0.03 & 9.42 \textpm 0.11\% & 27.78 \textpm 0.02 \\
  $h_0 = 0.7$ & 36.20 \textpm 1.31\% & 23.52 \textpm 0.28 & 80.09 \textpm 0.01\% & 13.99 \textpm 0.01 & 30.23 \textpm 4.19\% & 25.73 \textpm 0.77 & 9.02 \textpm 0.16\% & 27.78 \textpm 0.03 \\
  $h_0 = 0.8$ & 31.35 \textpm 2.65\% & 24.46 \textpm 0.54 & 80.09 \textpm 0.00\% & 14.00 \textpm 0.01 & 25.74 \textpm 0.32\% & 26.53 \textpm 0.06 & 9.29 \textpm 0.17\% & 27.77 \textpm 0.04 \\
  $h_0 = 0.9$ & 33.66 \textpm 1.61\% & 24.00 \textpm 0.30 & 80.10 \textpm 0.00\% & {\bf13.98 \textpm 0.01} & 31.08 \textpm 3.55\% & 25.56 \textpm 0.64 & 9.13 \textpm 0.18\% & 27.82 \textpm 0.04 \\
  $h_0 = 1$ & 34.69 \textpm 2.13\% & 23.75 \textpm 0.41 & 80.10 \textpm 0.00\% & 14.02 \textpm 0.01 & 25.34 \textpm 0.27\% & 26.61 \textpm 0.04 & 9.23 \textpm 0.13\% & 27.77 \textpm 0.03 \\
  \bottomrule
  \end{tabular}
  \caption{Coverage and cross-entropy (CE) after 5,000 iterations of DDS-KDE in all domains for each normalized bandwidth $h_0$. We report the mean over 10 trials (3 trials for Deceptive Maze) with error bars showing the standard error of the mean. Higher coverage and lower cross-entropy are better. \hl{Highlighted} cells are results from the main experiments in \autoref{fig:results}. The plots show the normalized bandwidth on the $x$-axis}
  \label{fig:bandwidth_ablation}
\end{figure*}

\begin{table}
\centering
\caption{Detailed hyperparameter specifications for each algorithm and domain.}\label{tab:hyperparameters}
\begin{tabular}{llcccc}
\toprule
Algorithm & Parameter & LP & Arm Repertoire & Deceptive Maze & Multi-feature LP \\
\toprule
\multirow{2}{*}{DDS-KDE} & Initial step size $\sigma_0$ & 1.5 & 0.5 & 1.5 & 1.5 \\
                         & Bandwidth $h$ & 25.6 & 10 & 0.01 & 2.56 \\
\midrule
\multirow{1}{*}{DDS-CNF} & Initial step size $\sigma_0$ & 0.5 & 1.5 & 1.5 & 0.5 \\
\midrule
\multirow{3}{*}{NS} & Initial step size $\sigma_0$ & 0.5 & 0.5 & 0.5 & 0.5 \\
                    & Nearest neighbors $k$ & 100 & 100 & 100 & 100 \\
                    & Acceptance threshold & $0.042 \cdot 512$ & $0.042 \cdot 200$ & $0.042$ & $0.042 \cdot 51.2$ \\
\midrule
\multirow{3}{*}{CMA-MAE} & Initial step size $\sigma_0$ & 0.5 & 0.2 & 0.5 & 0.5 \\
                         & Threshold min & 0 & 0 & 0 & 0 \\
                         & Learning rate $\alpha$ & 0.01 & 0.01 & 0.01 & 0.01 \\
\midrule
\multirow{1}{*}{CMA-ME} & Initial step size $\sigma_0$ & 0.5 & 0.2 & 0.5 & 0.5 \\
\midrule
\multirow{2}{*}{MAP-Elites (line)} & Isotropic variance $\sigma_\mathrm{iso}$ & 0.5 & 0.1 & 0.1 & 0.5 \\
                                   & Line variance $\sigma_\mathrm{line}$ & 0.2 & 0.2 & 0.2 & 0.2 \\
\bottomrule
\end{tabular}
\end{table}

\section{Conversion Formula for Affine Transformations of the Feature Space} \label{sec:affine}

We derive a formula that allows one to empirically find an optimal bandwidth on a domain, then apply the formula to compute the optimal bandwidths for any affine transformations of that domain.

Consider feature mappings $\vphi_1(\cdot)$ and $\vphi_2(\cdot)$.
Let $\vphi_2(\cdot)$ be an affine transformation of $\vphi_1(\cdot)$, i.e., $\vphi_2(\vx) = \mA\vphi_1(\vx) + \vb$.
Furthermore, assume there is a KDE $\hat{D}_{h_1}(\cdot)$ over $\vphi_1(\cdot)$.
We prove that when $\mA$ is some constant $c$, $\hat{D}_{h_2}$ on $\vphi_2(\cdot)$ where $h_2 = ch_1$ is an identical density function as $\hat{D}_{h_1}(\cdot)$, in expectation.

\begin{theorem}{\label{thm:bandwidth_selection}}
Consider $\hat{D}_{h_0, \mX}(\vx; \mathcal{X})$ and $\hat{D}_{h_1, \mY}(\vy; \mathcal{Y})$, where $\hat{D}_{h_0, \mX}$ is a KDE with bandwidth $h_0$ modeling the density of a random vector $\mX$ and $\hat{D}_{h_1, \mY}$ is a KDE with bandwidth $h_1$ modeling the density of a random vector $\mY$. If $\mY = a\mX + \vb$, then
$\hat{D}_{h_0, \mX}(\vx) = \hat{D}_{h_1, \mY}(\vy)$ when $h_0 = ah_1$.
\end{theorem}
\begin{proof}
The density function $p_X$ estimated by a KDE for a random vector $\mX$ with a given bandwidth $h$ is represented as $\hat{D}_{h, \mX}$.
Recall the definition of $\hat{D}_{h, \mX}$ according to the definition of a KDE from \autoref{eq:KDE}
\begin{equation}
\hat{D}_{h, \mX}(\vx; \mathcal{X}) = \frac{1}{|\mathcal{X}|h} \sum_{\vx' \in \mathcal{X}} K\left(\frac{\lVert\vx - \vx'\rVert}{h} \right)
\end{equation}
for kernel $K(\cdot)$ and $\mathcal{X} \subseteq \R^m$.
Consider the norm as a random variable $R' = \lVert\mX - \vx'\rVert$.
Since the kernel function $K(\cdot)$ maintains the property $\int_{-\infty}^{\infty} K_h(u)\,du = 1$~\citep{chen2017tutorial}, we can define the probability density function of $R'$ to be:
\begin{equation}
p_{R'}(r') = \frac{1}{h}K\left( \frac{r'}{h} \right) = \frac{1}{h}K\left(\frac{||\vx - \vx'||}{h}\right)
\end{equation}
Hence, the density estimate for $X$ can be represented in terms of $r_i = ||\vx - \vx_i||$ as shown below:
\begin{equation}
\hat{D}_{h, \mX}(\vx; \mathcal{X}) = \frac{1}{|\mathcal{X}|h} \sum_{\vx' \in \mathcal{X}} K\left(\frac{\lVert\vx - \vx'\rVert}{h} \right)
= \frac{1}{|\mathcal{X}|} \sum_{\vx' \in \mathcal{X}} p_{R'}(r')
\end{equation}
Now, when we transform our random vector $\mX$ into $\mY = a\mX + \vb$, where $a \in \R^+$ and $b \in \R^m$, we transform the norm $R'$ into $Q'$ as follows:
\begin{align*}
R' =  ||\mX - \vx'||, Q' & =  ||\mY - \vy'|| \\
& = ||a\mX + \vb - \left(a\vx' + \vb \right)|| \\
& = ||a\mX - a\vx'|| \\
& = a||\mX - \vx'|| \\
& = aR'
\end{align*}
So, we can represent $Q'$ as $Q' = g(R') = aR'$. Since $g$ is a monotonically increasing function, we can represent the density function of $Q'$ as follows:
\begin{align*}
p_{Q'}(q') & = p_{R'}\left(g^{-1}(q')\right)  \left|\frac{d}{dq'}g^{-1}(q')\right| \\
& = \left( \frac{1}{a} \right) \left( \frac{1}{h} \right) K\left( \frac{\frac{q'}{a}}{h} \right) \\
& = \frac{1}{ah} K\left( \frac{q'}{ah} \right) \\
& = \frac{1}{ah} K\left( \frac{||\vy - \vy'||}{ah} \right) 
\end{align*}
We can apply this relationship now to the density function estimated by the KDE for $\mY$, given that $q' = ||\mY - \vy'||$:
\begin{align*}
\hat{D}_{h, \mY}(\vy; \mathcal{Y}) & = \frac{1}{|\mathcal{Y}|}  \sum_{\vy' \in \mathcal{Y}} p_{q'}(q') \\
& = \frac{1}{|\mathcal{Y}|} \sum_{\vy' \in \mathcal{Y}} \frac{1}{ah} K\left( \frac{||\vy - \vy'||}{ah} \right) \\
& = \frac{1}{|\mathcal{Y}|(ah)} \sum_{\vy' \in \mathcal{Y}} K\left(\frac{||\vy - \vy'||}{ah} \right) \\
& = \hat{D}_{ah, \mY}(\vy; \mathcal{Y})
\end{align*}
Hence, if a KDE with a bandwidth of $h$ models the density of a given random vector $\mX$, then a random vector $\mY = a\mX + \vb$ should select a bandwidth of $ah$ to maintain a comparable density function. 
\end{proof}

\section{Reservoir Sampling for Maintaining a Representative Buffer}

As outlined in \autoref{sec:dds}, DDS maintains a buffer $\gB$ that stores the features of sampled solutions.
While it would be ideal to maintain an infinite buffer that stores all features ever encountered, due to runtime and memory constraints of KDE and CNF, we maintain a finite-sized buffer $\gB$ of size $|\gB| = b$.
Moreover, the subset of features retained by $\gB$ must accurately represent all the features ever encountered.
This can be described as the \emph{reservoir sampling} problem.

\begin{problem}[Reservoir Sampling~\citep{li1994reservior}]
Suppose there is a finite population of unknown size $N$. Derive an algorithm that draws a simple random sample of size $b$ from this population by scanning it exactly once, with the memory constraint such that the algorithm can remember at most $b$ samples at a time.
\end{problem}

The population consists of the features that are discovered by DDS throughout its entire execution.
Since the buffer $\gB$ is a simple random sample of all features, the density function estimated on $\gB$ will be close to the density function of all features.
Thus, with reservoir sampling, we are only required to maintain a subset of all discovered features without severely altering the shape of the density estimation of the feature space.
In DDS, this is achieved by leveraging algorithm L, the optimal reservoir sampling algorithm~\citep{li1994reservior}.

\section{Proofs} \label{sec:proofs}

\begin{theorem}
    A kernel density estimate $\hat{D}_h(\vx;\gB)$ managed with reservoir sampling, such that features in the buffer are exchanged one at a time, is $\frac{1}{|\gB|h}$-uniformly stable, where $h$ is the bandwidth.
\end{theorem}
\begin{proof}
Let buffer $\gB \subseteq \R^m$ and $\gB' = \gB \setminus \{\vx_i\} \cup \{\vx_j\}$ for $\vx_i \in \gB$ and $\vx_j \in \R^m$.
Consider the maximum difference between $\hat{D}_h(\vx;\gB)$ and $\hat{D}_h(\vx;\gB')$:
\begin{align}
\sup_{\vx \in \R^d} \left|\hat{D}_h(\vx;\gB) - \hat{D}_h(\vx;\gB')\right| &=  \sup_{\vx \in \R^m} \left|\frac{1}{|\mathcal{\gB}|h}\sum_{\va \in \gB} K\left(\frac{\lVert \vx - \va \rVert}{h}\right) - \frac{1}{|\mathcal{\gB}|h}\sum_{\va' \in \gB'} K\left(\frac{\lVert \vx-\va' \rVert}{h}\right)\right| \\
&= \frac{1}{|\mathcal{\gB}|h} \sup_{\vx \in \R^m} \left|\sum_{\va \in \gB} K\left(\frac{\lVert \vx-\va \rVert}{h}\right) - \sum_{\va' \in \gB'} K\left(\frac{\lVert \vx-\va' \rVert}{h}\right)\right| \label{eq:cancel}\\
&= \frac{1}{|\mathcal{\gB}|h} \sup_{\vx \in \R^m} \left|K\left(\frac{\lVert \vx-\vx_i \rVert}{h}\right) - K\left(\frac{\lVert \vx-\vx_j \rVert}{h}\right)\right| \\
&\le \frac{1}{|\mathcal{\gB}|h} \label{eq:last}
\end{align}
Since $\gB$ and $\gB'$ differ by exactly one element, each term in the summation for \autoref{eq:cancel} cancels except for $\vx_i$ and $\vx_j$.
\autoref{eq:last} is due to the fact that any kernel $K(\cdot)$ is bounded between $[0,1]$ by definition.
\end{proof}

\begin{theorem}
    Novelty score $\rho(\vx;\gB)$ is $\frac{W}{k}$-uniformly stable, where $k$ is the nearest neighbors parameter in novelty score and $W = \max_{\vs_1,\vs_2 \in S} \lVert \vs_1 - \vs_2 \rVert$ is the diameter of the feature space $S$.
\end{theorem}

\begin{proof}

Let buffer $\gB \subseteq \R^m$ and $\gB' = \gB \cup \{\vx_j\}$ for $\vx_j \in \R^m$.
Since we constructed $\gB'$ by adding a single element to $\gB$, $N_k(\vx, \gB)$ and $N_k(\vx, \gB')$ differ by at most one element $\vx_i \in \gB$ and $\vx_j \in \gB'$.
Hence,
\begin{align*}
\sup_{\vx \in \R^d} \left|\rho(\vx;\gB) - \rho(\vx;\gB')\right| &=
\sup_{\vx\in\R^d} \left|\frac{1}{k} \sum_{\va \in N_k(\vx;\gB)} \lVert \vx - \va\rVert - \frac{1}{k} \sum_{\va \in N_k(\vx;\gB')} \lVert \vx - \va\rVert\right| \\
&=  \frac{1}{k} \sup_{\vx\in\R^d} \left|\lVert \vx - \vx_i\rVert - \lVert \vx - \vx_j\rVert\right| \\
&\le \frac{1}{k} \sup_{\vx\in\R^d} \lVert \vx_i - \vx_j \rVert \\
&\le \frac{W}{k}
\end{align*}
\end{proof}

\begin{theorem}
Let $\gB \subseteq \R^m$ be a set of features.
Consider the rankings $\pi_\mathrm{NS}$ and $\pi_\mathrm{KDE}$ on another set of features $\{\vphi_1, \dots, \vphi_m\} \subseteq \R^m$ where $\vphi_i = \vphi(\vtheta_i)$.
We define the rankings as follows: $\pi_\mathrm{NS}(i) \ge \pi_\mathrm{NS}(j) \iff \rho(\vphi_i; \gB) \ge \rho(\vphi_j; \mathcal{X})$ and $\pi_\mathrm{KDE}(i) \ge \pi_\mathrm{KDE}(j) \iff \hat{D}(\vphi_i; \gB) \le \hat{D}(\vphi_j; \gB)$.
We show that $\pi_{NS} = \pi_{KDE}$, when NS has $k=\infty$ (or, equivalently, $k = |\gB|$) and KDE has triangular kernel $K(u) = 1 - |u|$ with support over the entire feature space $S$.
\end{theorem}
\begin{proof}
Recall the novelty score function,
\begin{equation*}
\rho(\vx; \gB) = \frac{1}{|\gB|} \sum_{\vb \in \gB} \lVert\vx - \vb\rVert
\end{equation*}
Consider the kernel density estimator with $h = 1$ and triangular kernel,
\begin{align*}
\hat{D}(\vx; \gB) &= \frac{1}{|\gB|} \sum_{\vb \in \gB} K\left(\lVert \vx - \vb\rVert\right) \\
                          &= \frac{1}{|\gB|} \sum_{\vb \in \gB} (1 - \lVert \vx - \vb\rVert) \\
                          &= 1 - \frac{1}{|\gB|} \sum_{\vb \in \gB} \lVert \vx - \vb\rVert \\
                          &= 1 - \rho(\vx; \gB)
\end{align*}
Thus, $\rho(\vphi_i; \gB) \le \rho(\vphi_j; \gB) \iff  \hat{D}(\vphi_i; \gB) \ge \hat{D}(\vphi_j; \gB)$.
\end{proof}

\newpage

\section{Statistical Analysis} \label{sec:stats}
We include results from our statistical analysis in \autoref{sec:experiments}.
We performed one-way ANOVAs in each domain, shown in \autoref{tab:anova}.
We also performed pairwise comparisons with Tukey's HSD, shown in \autoref{tab:pairwise_coverage} and \autoref{tab:pairwise_ce}.

\begin{table}[thpb]
  \caption{
    One-way ANOVA results in each domain. All $p$-values are less than 0.001.
  }
  \label{tab:anova}
  \centering
  {
  \setlength{\tabcolsep}{3pt}
  \fontsize{7.8}{7.1}\selectfont
  \begin{tabular}{lllll}
  \toprule
  {}              & \multicolumn{1}{c}{Coverage} & \multicolumn{1}{c}{Cross-Entropy} \\
  \midrule
  LP              & $F(5, 54) = 206.61$  & $F(5, 54) = 156.04$  \\
  Arm Repertoire  & $F(5, 54) = 9691.05$  & $F(5, 54) = 1523.35$  \\
  Deceptive Maze  & $F(5, 54) = 23.89$ & $F(5, 54) = 39.33$  \\
  Multi-feature LP  & $F(5, 54) = 3350$ & $F(5, 54) = 466.49$ \\
  \bottomrule
  \end{tabular}
  }
\end{table}

\begin{table}[thpb]
  \caption{Pairwise comparisons for coverage in each domain. Each entry compares the method in the row to the method in the column. For instance, we can see that DDS-KDE had significantly higher coverage than NS in LP. Comparisons were performed with Tukey's HSD test. $>$ indicates significantly greater, $<$ indicates significantly less, and $-$ indicates no significant difference. $\varnothing$ indicates an invalid comparison.}
  \centering
  \label{tab:pairwise_coverage}
  \setlength{\tabcolsep}{2pt}
  \fontsize{7.8}{7.1}\selectfont
  \begin{tabular}{lrrrrrrrrrrrrrrrrrrrrrrrr}
  \toprule
   & \multicolumn{6}{c}{LP} & \multicolumn{6}{c}{Arm Repertoire} & \multicolumn{6}{c}{Deceptive Maze} & \multicolumn{6}{c}{Multi-dim LP} \\
   \cmidrule(r){2-7}
   \cmidrule(r){8-13}
   \cmidrule(r){14-19}
   \cmidrule(r){20-25}
   & \rot{DDS-KDE} & \rot{DDS-CNF} & \rot{NS} & \rot{CMA-MAE} & \rot{CMA-ME} & \rot{MAP-Elites (line)} & \rot{DDS-KDE} & \rot{DDS-CNF} & \rot{NS} & \rot{CMA-MAE} & \rot{CMA-ME} & \rot{MAP-Elites (line)} & \rot{DDS-KDE} & \rot{DDS-CNF} & \rot{NS} & \rot{CMA-MAE} & \rot{CMA-ME} & \rot{MAP-Elites (line)} & \rot{DDS-KDE} & \rot{DDS-CNF} & \rot{NS} & \rot{CMA-MAE} & \rot{CMA-ME} & \rot{MAP-Elites (line)} \\
  \midrule
  DDS-KDE & $\varnothing$ & $-$ & $>$ & $>$ & $>$ & $>$ & $\varnothing$ & $-$ & $>$ & $>$ & $>$ & $>$ & $\varnothing$ & $>$ & $>$ & $-$ & $>$ & $-$ & $\varnothing$ & $>$ & $>$ & $>$ & $>$ & $>$ \\
  DDS-CNF & $-$ & $\varnothing$ & $>$ & $>$ & $>$ & $>$ & $-$ & $\varnothing$ & $>$ & $>$ & $>$ & $>$ & $<$ & $\varnothing$ & $>$ & $<$ & $-$ & $-$ & $<$ & $\varnothing$ & $>$ & $>$ & $>$ & $>$ \\
  NS & $<$ & $<$ & $\varnothing$ & $<$ & $<$ & $<$ & $<$ & $<$ & $\varnothing$ & $>$ & $>$ & $>$ & $<$ & $<$ & $\varnothing$ & $<$ & $-$ & $<$ & $<$ & $<$ & $\varnothing$ & $<$ & $>$ & $>$ \\
  CMA-MAE & $<$ & $<$ & $>$ & $\varnothing$ & $>$ & $>$ & $<$ & $<$ & $<$ & $\varnothing$ & $>$ & $>$ & $-$ & $>$ & $>$ & $\varnothing$ & $>$ & $>$ & $<$ & $<$ & $>$ & $\varnothing$ & $>$ & $>$ \\
  CMA-ME & $<$ & $<$ & $>$ & $<$ & $\varnothing$ & $-$ & $<$ & $<$ & $<$ & $<$ & $\varnothing$ & $<$ & $<$ & $-$ & $-$ & $<$ & $\varnothing$ & $<$ & $<$ & $<$ & $<$ & $<$ & $\varnothing$ & $<$ \\
  MAP-Elites (line) & $<$ & $<$ & $>$ & $<$ & $-$ & $\varnothing$ & $<$ & $<$ & $<$ & $<$ & $>$ & $\varnothing$ & $-$ & $-$ & $>$ & $<$ & $>$ & $\varnothing$ & $<$ & $<$ & $<$ & $<$ & $>$ & $\varnothing$ \\
  \bottomrule
  \end{tabular}
\end{table}

\begin{table}[thpb]
  \caption{Pairwise comparisons for cross-entropy in each domain. Note that lower cross-entropy is better, so significantly less ($<$) indicates that a method is significantly better.}
  \centering
  \label{tab:pairwise_ce}
  \setlength{\tabcolsep}{2pt}
  \fontsize{7.8}{7.1}\selectfont
  \begin{tabular}{lrrrrrrrrrrrrrrrrrrrrrrrr}
  \toprule
   & \multicolumn{6}{c}{LP} & \multicolumn{6}{c}{Arm Repertoire} & \multicolumn{6}{c}{Deceptive Maze} & \multicolumn{6}{c}{Multi-dim LP} \\
   \cmidrule(r){2-7}
   \cmidrule(r){8-13}
   \cmidrule(r){14-19}
   \cmidrule(r){20-25}
   & \rot{DDS-KDE} & \rot{DDS-CNF} & \rot{NS} & \rot{CMA-MAE} & \rot{CMA-ME} & \rot{MAP-Elites (line)} & \rot{DDS-KDE} & \rot{DDS-CNF} & \rot{NS} & \rot{CMA-MAE} & \rot{CMA-ME} & \rot{MAP-Elites (line)} & \rot{DDS-KDE} & \rot{DDS-CNF} & \rot{NS} & \rot{CMA-MAE} & \rot{CMA-ME} & \rot{MAP-Elites (line)} & \rot{DDS-KDE} & \rot{DDS-CNF} & \rot{NS} & \rot{CMA-MAE} & \rot{CMA-ME} & \rot{MAP-Elites (line)} \\
  \midrule
  DDS-KDE & $\varnothing$ & $-$ & $<$ & $<$ & $<$ & $<$ & $\varnothing$ & $>$ & $<$ & $>$ & $<$ & $>$ & $\varnothing$ & $<$ & $<$ & $-$ & $<$ & $<$ & $\varnothing$ & $<$ & $<$ & $<$ & $<$ & $<$ \\
  DDS-CNF & $-$ & $\varnothing$ & $<$ & $-$ & $<$ & $<$ & $<$ & $\varnothing$ & $<$ & $>$ & $<$ & $-$ & $>$ & $\varnothing$ & $<$ & $>$ & $-$ & $-$ & $>$ & $\varnothing$ & $-$ & $<$ & $-$ & $<$ \\
  NS & $>$ & $>$ & $\varnothing$ & $>$ & $>$ & $>$ & $>$ & $>$ & $\varnothing$ & $>$ & $<$ & $>$ & $>$ & $>$ & $\varnothing$ & $>$ & $-$ & $>$ & $>$ & $-$ & $\varnothing$ & $<$ & $-$ & $<$ \\
  CMA-MAE & $>$ & $-$ & $<$ & $\varnothing$ & $<$ & $<$ & $<$ & $<$ & $<$ & $\varnothing$ & $<$ & $<$ & $-$ & $<$ & $<$ & $\varnothing$ & $<$ & $<$ & $>$ & $>$ & $>$ & $\varnothing$ & $>$ & $<$ \\
  CMA-ME & $>$ & $>$ & $<$ & $>$ & $\varnothing$ & $-$ & $>$ & $>$ & $>$ & $>$ & $\varnothing$ & $>$ & $>$ & $-$ & $-$ & $>$ & $\varnothing$ & $>$ & $>$ & $-$ & $-$ & $<$ & $\varnothing$ & $<$ \\
  MAP-Elites (line) & $>$ & $>$ & $<$ & $>$ & $-$ & $\varnothing$ & $<$ & $-$ & $<$ & $>$ & $<$ & $\varnothing$ & $>$ & $-$ & $<$ & $>$ & $<$ & $\varnothing$ & $>$ & $>$ & $>$ & $>$ & $>$ & $\varnothing$ \\
  \bottomrule
  \end{tabular}
\end{table}

\newpage

\section{Coverage Map in Feature Space} \label{sec:heatmaps}

\begin{figure*}[htbp!]
    \centering
    \includegraphics[width=\linewidth]{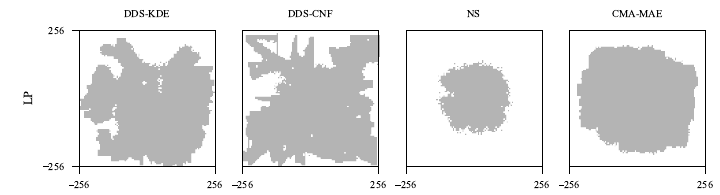}
    \includegraphics[width=\linewidth]{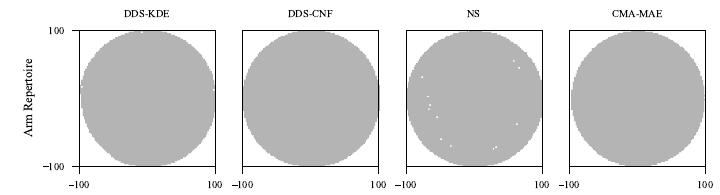}
    \includegraphics[width=\linewidth]{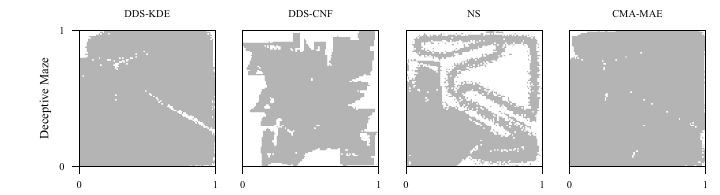}
    \caption{Heatmaps of DDS-KDE, DDS-CNF, NS, and CMA-MAE in the three domains that have 2D feature spaces (LP, Arm Repertoire, Deceptive Maze).} \label{fig:heatmaps}
\end{figure*}

\end{document}